\documentclass{tran-l}

\usepackage{cite}
\usepackage{amsmath,amssymb,amsfonts}
\usepackage[ruled]{algorithm}
\usepackage{algorithmic}
\usepackage{graphicx}
\usepackage{textcomp}
\usepackage{diagbox}
\usepackage{xcolor}
\usepackage{xspace}
\usepackage{dsfont}

\newcommand{\R}{\mathbb{R}}
\newcommand{\N}{\mathbb{N}}

\newcommand{\norm}[1]{\lVert#1\rVert}

\newcommand*\diff{\mathop{}\!\mathrm{d}}

\newcommand{\var}{\mathrm{var}}

\newtheorem{lemma}{Lemma}
\newtheorem{theorem}{Theorem}

\newtheorem{remark}{Remark}
\newtheorem{assumption}{Assumption}

\numberwithin{equation}{section}

\begin{document}

\title{Heat Death of Generative Models in Closed-Loop Learning}


%    author one information
\author{Matteo Marchi}
\address{Matteo Marchi}
\curraddr{}
\email{matmarchi@ucla.edu}
\thanks{}

%    author two information
\author{Stefano Soatto}
\address{Stefano Soatto}
\curraddr{}
\email{soatto@ucla.edu}
\thanks{}

\author{Pratik Chaudhari}
\address{Pratik Chaudhari}
\curraddr{}
\email{pratikac@seas.upenn.edu}
\thanks{}

\author{Paulo Tabuada}
\address{Paulo Tabuada}
\curraddr{}
\email{tabuada@ee.ucla.edu}
\thanks{}

\date{}

\dedicatory{}

%    Abstract is required.
\begin{abstract}
    Improvement and adoption of generative machine learning models is rapidly accelerating, as exemplified by the popularity of LLMs (Large Language Models) for text, and diffusion models for image generation.As generative models become widespread, data they generate is incorporated into shared content through the public web. This opens the question of what happens when data generated by a model is fed back to the model in subsequent training campaigns. This is a question about the stability of the training process, whether the distribution of publicly accessible content, which we refer to as ``knowledge'', remains stable or collapses.
    
    Small scale empirical experiments reported in the literature show that this closed-loop training process is prone to degenerating. Models may start producing gibberish data, or sample from only a small subset of the desired data distribution (a phenomenon referred to as mode collapse). So far there has been only limited theoretical understanding of this process, in part due to the complexity of the deep networks underlying these generative models.
    
    The aim of this paper is to provide insights into this process (that we refer to as ``generative closed-loop learning'') by studying the learning dynamics of generative models that are fed back their own produced content in addition to their original training dataset. The sampling of many of these models can be controlled via a ``temperature'' parameter. Using dynamical systems tools, we show that, unless a sufficient amount of external data is introduced at each iteration, any non-trivial temperature leads the model to asymptotically degenerate. In fact, either the generative distribution collapses to a small set of outputs, or becomes uniform over a large set of outputs.
\end{abstract}

\maketitle

%%%%%%%%%%%%%%%%%%%%%%%%%%%%%%%%%%%%%%%%%%%%%%%%%%%%%%%%%%%%%%%%%%%%%%%%%%%%%%%%
\section{Introduction}
Generative models have exploded in popularity in recent years, primarily driven by the adoption of diffusion models~\cite{cao2024survey} for image generation, and so-called LLMs (Large Language Models)~\cite{zhao2023survey} for textual generation. With this explosion, came renewed concerns about AI, especially tied to the \emph{generative} nature of these models.
% {\color{red} Enormously sized}
{\color{black} Large scale} neural networks underlie most of these models,{\color{black} including, for example, Llama 2 which} is trained on 2 trillion tokens~\cite{touvron2023llama}. As these models generate data that is published on the internet, they pollute their own training datasets with synthetic data, possibly leading to a spiraling decay of the quality of these models and of the internet by extension.

We are concerned with the setting where a generative model is iteratively trained, and the outcome of each iteration is dependent on the current data distribution encoded by the model (typically by including samples generated by the model in the training set). Serious concerns about decay of such a training process arose first in GANs (Generative Adversarial Networks), where the problem of ``mode collapse''~\cite{thanh2020catastrophic} was identified. Analogous issues seem to be a general feature of violating distributional assumptions about the training dataset, even for non-GAN models. One way of framing such violations is as ``data poisoning''~\cite{biggio2012poisoning}, a problem that is likely to become more common, as models trained from public domain internet data are especially susceptible to data poisoning attacks~\cite{carlini2023poisoning}. This is also related to the notion of ``distribution shift''~\cite{koh2021wilds}, although most existing work focuses on the distribution shift occurring at test time and coming from an external source. In our setting the shift occurs at \emph{training time} and has an \emph{internal} origin.

As this is such a new development, there is still only partial understanding of the phenomenon, and much published work is empirical in nature. In~\cite{martinez2023combining} and~\cite{martinez2023towards}, the authors train image diffusion models, iteratively including synthetic samples, and show significant degradation of the quality of the produced images. In~\cite{shumailov2023curse}, it is shown, both theoretically and experimentally, that generative Gaussian models undergo degenerative collapse. A case of closed-loop learning when the sampling of the model is biased (samples may be taken closer to the mean) under a variety of synthetic data policies is studied in~\cite{alemohammad2023self}. In their results, non-degeneration could be ensured only by introducing a sufficient fraction of fresh data at each training iteration. This aligns with the results in~\cite{bertrand2023stability}, where the authors establish (theoretically and experimentally) that maintaining a high enough fraction of fresh data is a sufficient condition to prevent degeneration.

Most generative models include a way to modulate their sampling probabilities through ``temperature'', typically as a way to make the outputs more or less random. In this work, we focus on the effect of temperature on the closed-loop learning dynamics of generative models, a perspective that received little attention so far. In particular: 1) We define a class of ``generative closed-loop learning models with temperature'' that captures many real-world scenarios. 2) We perform a theoretical analysis of the resulting closed-loop learning dynamics, and establish that modulating sampling with temperature leads to degeneration of the learning process. 3) We characterize the type of degeneration depending on one of three possible temperature regimes. As the models degenerate (for any amount of temperature modulation), so do their datasets, consequently losing any knowledge they originally contained, {\color{black} if not explicitly preserved and re-introduced}. When applied to the internet, this predicts that  {\color{black} unless a copy of the pre-generative-models Internet is preserved,} eventually no model will be able to be trained effectively using the internet as a data source.
Our results share some similarities with~\cite{alemohammad2023self, bertrand2023stability}, and are compatible with their conclusions, but in contrast to those papers we use tools and techniques from dynamical and control systems for the analysis.

\section{Notation}
\begin{itemize}
    \item We denote by $e_i$ the $i$-th element of the standard basis of $\R^n$, i.e., the vector of all zeroes except for a one in its $i$-th entry.

    \item The symbol $\mathbf{1}$ denotes the vector $x\in\R^n$ with all elements equal to one.
    
    \item We define $\Delta^n$ as the $n$-dimensional probability simplex: $$\Delta^n = \left\{x \in \R^n ~|~ \sum_{i=1}^n x_i = 1, x_i \ge 0\right\},$$ and its restriction to strictly positive probabilities as $\Delta^n_{>0}$. An element of $\Delta^n$ is called a ``probability vector''. The boundary of the simplex is denoted by $\partial \Delta^n$. 

    \item Given some $X\in\Delta^n$, we say that the random variable $Y$ is sampled according to $X$, or $Y\sim X$, to mean that for all $i\in\left\{1,2,\dots,n\right\}$ we have $ \operatorname P\left(Y = e_i\right) = X_i. $

    \item If $X(k)$ is a stochastic process, $\mathcal F_k$ denotes the filtration adapted to the stochastic process up to time $k$. We say an event happens a.s. to mean ``almost surely'', i.e., with probability 1 (w.p. 1).

    \item Unless otherwise noted, $\norm{\cdot}$ denotes the usual vector 2-norm over $\R^n$, and $d(x,y) = \norm{x-y}$ with $x,y\in\R^n$ is the distance between $x$ and $y$. If one of the arguments is a set $\Omega\subseteq\R^n$, it denotes the distance from a point to that set $d(x, \Omega) = \inf_{y\in \Omega} d(x,y)$.

    \item The notation $f(x)\xrightarrow[x\to a]{} \Omega$, with $\Omega$ a set means $\lim_{x\to a} d\left(f(x), \Omega\right) = 0$.

    \item We normally use capitalized letters to denote random variables, and lower-case when they are deterministic, or when the randomness is not relevant (i.e., $X$ vs $x$).

    \item A continuous function $\alpha:[0,a)\to\R_{\ge 0}$, with $a\in\R_{\ge 0}\cup \{+\infty\}$, is said to be of class kappa ($\alpha\in\mathcal K$) if it is strictly increasing and $\alpha(0)=0$.
\end{itemize}

\section{Generative Closed-Loop Learning}\label{sec:model}
We describe a generative model as a parameterized family of probability distributions over a finite set of $n\in\N$ possible elements $\mathcal Y = \left\{\mathcal Y_1, \mathcal Y_2, \dots, \mathcal Y_n\right\}$\footnote{This family is a subset of the set of categorical distributions over $n$ categories. In general we do not require that the family of distributions expressible by the model is the full set of categorical distributions, which is unrealistic for very high $n$ (for example, if the outcomes $\mathcal Y$ are rgb-images).}. We denote such family as $\phi:\R^p\to\Delta^n$, a map from a parameter vector $w\in\R^p$ to a probability vector $\phi(w)\in\Delta^n$ for the elements $\mathcal Y$. These are the outputs that the model can generate when sampled. Without loss of generality, we identify each $\mathcal Y_i$ with the $i$-th vector of the standard basis of $\R^n$. These elements can be interpreted differently depending on the specific generative model, e.g., for a language model each $\mathcal Y_i$ could be a word, token, sentence, or sentence class from a large but finite set.

\subsection{Model sampling with temperature control}
For a trained generative model, letting $Y$ be the output of the model when sampled, we denote by $\Theta = \phi(w)$ the ``nominal'' probability of generating each of the possible elements of $\mathcal Y$.  Specifically, the probability of generating the $i$-th element corresponds to the $i$-th entry of the vector $\Theta$. However, when sampled, the actual generation probabilities are filtered through a \emph{temperature} function \mbox{$\tau : \Delta^n \to \Delta^n$}. Therefore, for $i\in\left\{1,2,\dots,n\right\}$, the sampled output $Y\in\mathcal Y$ satisfies:
\begin{equation}
    P\left(Y = \mathcal Y_i\right) = \tau\left(\Theta\right)_i,
\end{equation}
where a subscripted index $i$ denotes the $i$-th vector element. For our results to hold, we require the temperature function to satisfy some assumptions (see Sec.~\ref{sec:concrete}). We show that these assumptions hold for the temperature function induced by the \emph{softmax} operation, typically used in deep learning.

\subsection{Learning process}
{\color{black} We use the term ``generative closed-loop learning'', or just ``closed-loop learning'', to refer to a generative model trained on data that includes its own output from prior runs.}
When a generative model learns from its own output, the probability vector $\Theta$ becomes a stochastic process $\Theta(k)$ evolving over (discrete) time $k\in\mathbb Z_{\ge 0}$. We assume that a model is initially trained on some externally provided dataset of some size $\ell\in\N$:
$$D(\ell) = \left\{Y(1), Y(2), \dots, Y(\ell)\right\},$$
where we use $Y(k)$ with $k\le \ell$ to denote the externally provided data, i.e.,  training only starts at time $k=\ell$. Similarly, for each time $k \ge \ell$ we have a parameter vector $w(k)$ and its associated probability vector $\Theta(k) = \phi(w(k))$. Finally, let the training be represented by a (in general stochastic) function\footnote{While the retraining function here takes the current parameters as an argument, it can also represent a form of retraining where the model is ``reset'' and trained from scratch over a new dataset by ignoring $w$.} $f$ that maps a parameter vector $w$ and training data $D$ to a ``retrained'' parameter $ f(w, D)$. Then, for each time $k\in\mathbb Z_{\ge 0}$, $k \ge \ell$, the closed-loop learning stochastic process unfolds as follows:
\begin{equation}\label{eq:self-learning_model}
    \begin{aligned}
        Y(k+1) &\sim \tau\left(\Theta(k)\right)\\
        D(k+1) &= D(k) \cup \{Y(k+1)\}\\
        w(k+1) &= f(w(k), D(k+1))\\
        \Theta(k+1) &= \phi(w(k+1)),
    \end{aligned}
\end{equation}
where $D(k) \cup \{Y(k+1)\}$ models ``adding'' the generated output sample to the current set of training data\footnote{For notational simplicity, in~\eqref{eq:self-learning_model}, the process retrains the model after each generated sample, however our results hold even in the case where some variable but bounded number of samples $N \ge 1$ is generated and added to the dataset before retraining.}. For some initial dataset of $\ell$ samples, the process has the initial conditions $w(\ell) = f(w_0, D(k))$, and $\Theta(\ell) = \phi(w(\ell))$ for some $w_0\in\R^p$. The recursive process~\eqref{eq:self-learning_model} induces a probability distribution for each $\Theta(k)$. Like for the temperature function $\tau$, in Sec.~\ref{sec:concrete} we will require that the process~\eqref{eq:self-learning_model} satisfies some general properties.

\subsection{Problem statement}
Given the closed-loop learning process~\eqref{eq:self-learning_model}, we want to know what are the long term properties of the probability vector $\Theta(k)$, i.e.,  what is the asymptotic behavior of $\Theta(k)$ as time increases?  
Since $\Theta(k)$ describes the probability of generated data, the asymptotic behavior of $\Theta(k)$ determines the ultimate composition of the dataset $D(k)$ as well. For example, if $\Theta(k)$ were to converge to a point independent of the initial dataset $D(\ell)$, any initial knowledge encoded by the dataset is eventually lost.

\section{A common class of models}\label{sec:concrete}
In the previous section we presented an abstracted notion of closed-loop learning. We now give specific conditions on the temperature function $\tau$ and the behavior of the training algorithm represented by $f$ and $\phi$ in~\eqref{eq:self-learning_model}, and show that they are realistic for common closed-loop learning models.

\subsection{Temperature}
We assume that the class of temperature functions $\tau$ as defined in Sec.~\ref{sec:model} satisfies a few properties:
\begin{assumption}\label{ass:temperature}
    The temperature function $\tau:\Delta^n\to\Delta^n$ in~\eqref{eq:self-learning_model} is assumed to satisfy the following properties:
    \begin{enumerate}
        \item It is continuous and strictly element-wise order preserving, i.e., for any $\theta\in\Delta^n$ and $i,j\in\{1,2,\dots,n\}$:
        \begin{equation}\label{eq:order}
            \begin{gathered}
                \theta_i < \theta_j \implies \tau(\theta_i) < \tau(\theta)_j\\
                \theta_i = \theta_j \implies \tau(\theta_i) = \tau(\theta_j).
            \end{gathered}
        \end{equation}

        \item Given an index set $I\subseteq\{1,2,\dots,n\}$, let $V_I:\Delta^n\to\R_{\ge 0}$ be defined as\footnote{$V_I$ will be used as a Lyapunov function later in the analysis.}:
        \begin{equation}\label{eq:lyapunov}
            V_I(\theta) = \max_{i\in I} \left\{\frac{\theta_i}{\sum_{j\in I}\theta_j}\right\}-\min_{i\in I}\left\{\frac{\theta_i}{\sum_{j\in I}\theta_j}\right\}.
        \end{equation}    
        Then, $\tau$ satisfies exactly one of the properties:
        \begin{enumerate}
            \item It is the identity: $\tau(\theta) = \theta$.
            \item For any index set $I\subseteq\{1,2,\dots,n\}$ where $\min_{i\in I}\theta_i > 0$ and $\max_{i\in I}\theta_i > \min_{i\in I}\theta_i$:
            $$V_I(\tau(\theta)) - V_I(\theta) < 0.$$
            \item For any index set $I\subseteq\{1,2,\dots,n\}$ where $\min_{i\in I}\theta_i > 0$ and $\max_{i\in I}\theta_i > \min_{i\in I}\theta_i$:
            $$V_I(\tau(\theta)) - V_I(\theta) > 0.$$
        \end{enumerate}
    \end{enumerate}
\end{assumption}
Intuitively, case 2.b represents a ``contracting'' $\tau$, and case 2.c an ``expanding'' $\tau$. The function $V_I$ quantifies how close a probability vector $\theta\in\Delta^n$ is to uniform when conditioned to a specific subset of variables. Note that when an index set $I$ includes all non-zero elements of $\theta$, the renormalizing term $\sum_{j\in I}\theta_j$ in~\eqref{eq:lyapunov} is equal to one, and~\eqref{eq:lyapunov} reduces to $V_I(\theta) = \max_{i\in I} \theta_i - \min_{i\in I}\theta_i$.

The notion of temperature typically used in generative models satisfies the requirements listed above. In fact, this is the case for the \emph{softmax} temperature, as we now show. Many machine learning models do not directly output a set of probabilities, but a vector of so-called \emph{logits} $z\in\R^n$ that is converted into a probability vector via the $\operatorname{softmax}$ function and a positive temperature parameter $T>0$ as follows:
\begin{equation}\label{eq:softmax}
\begin{aligned}
    \theta_T &=  \operatorname{softmax}\left(zT^{-1}\right)\\
    &= \frac{1}{\sum_{i=1}^n \exp\left(\frac{z_i}{T}\right)}
    \begin{bmatrix}
        \exp\left(\frac{z_1}{T}\right) & \dots & \exp\left(\frac{z_n}{T}\right)
    \end{bmatrix}^{\top}.
\end{aligned}
\end{equation}
Note that $\tau$, as defined in Sec.~\ref{sec:model}, is a map between probability vectors, but~\eqref{eq:softmax} maps logits to probabilities. Consider a logit vector $z\in\R^n$. While~\eqref{eq:softmax} is not a valid $\tau$, it induces a unique map transforming $\theta = \operatorname{softmax}(z)$ (the ``nominal'' probabilities associated to $z$) to $\theta_T = \operatorname{softmax}\left(T^{-1}z\right)$ (the ``temperature filtered'' probabilities associated to the same $z$). Defining $Z=\sum_{i=1}^n \exp(z_i)$ we have:
\begin{equation}\label{eq:softmax_tau}
    \begin{aligned}
        \theta_T &= \operatorname{softmax}\left(zT^{-1}\right)\\
        &= \operatorname{softmax}\left(T^{-1}\left(\begin{bmatrix}
            \ln\left(\theta_1\right) & \dots & \ln\left(\theta_n\right)
        \end{bmatrix}^\top + \mathbf 1\ln(Z)\right)\right)\\
        &= \operatorname{softmax}\left(T^{-1}\begin{bmatrix}
            \ln\left(\theta_1\right) & \dots & \ln\left(\theta_n\right)
        \end{bmatrix}^\top\right)\\
        &= \frac{1}{\sum_{i=1}^n \theta^{\frac{1}{T}}}\begin{bmatrix}
            \theta_1^{\frac{1}{T}} & \dots & \theta_n^{\frac{1}{T}}
        \end{bmatrix}^\top,
    \end{aligned}
\end{equation}
where the third equality holds because $\operatorname{softmax}$ is invariant to addition of the same constant ($T^{-1}\ln(Z)$) to all input entries. This induced map $\tau(\theta) = \theta_T$ satisfies exactly one of the three requirements previously introduced. We formalize this in the following Lemma:
\begin{lemma}
    Consider the function $\tau:\Delta^n\to\Delta^n$ defined by $\tau(\theta) = \operatorname{softmax}\left(T^{-1}\ln(\theta)\right)$, with $T\in\R_{>0}$. The function $\tau$ satisfies Assumption~\ref{ass:temperature}. In particular, it satisfies properties 2.a, 2.b, and 2.c for $T=1$, $T>1$, and $T<1$ respectively.
\end{lemma}
\begin{proof}
    Since $\tau$ takes the form~\eqref{eq:softmax_tau}, and $x\mapsto x^{\frac{1}{T}}$ is a continuous strictly monotone increasing function, it is immediate that $\tau$ is also continuous (the denominator in~\eqref{eq:softmax_tau} is always bounded away from zero) and preserves the order of the elements of $\theta$. Further, note that if $\theta_i=0$, then $\tau(\theta)_i = 0$. 
    
    For the case $T=1$, it is immediate to see that $\tau$ becomes the identity function (Note that $\sum_{i=1}^n \theta_i = 1$). 
    
    For the cases where $T \ne 1$, let $I\subseteq\{1,2,\dots,n\}$ be as in Assumption~\ref{ass:temperature}, then, $V_I(\tau(\theta)) < V_I(\theta)$. To see that this is true, let $M,m\in I$ be respectively the (not necessarily unique) indices of the greatest and smallest non-zero elements of $\theta$, and consider the derivative with respect to the temperature parameter $T$ of $V_I(\tau(\theta))$:
    \begin{equation*}
        \begin{aligned}
            \frac{\partial}{\partial T}V_I(\tau(\theta)) &= \frac{\partial}{\partial T} \left\{\frac{\theta^{\frac{1}{T}}_M}{\sum_{j\in I}\theta_j^{\frac{1}{T}}}-\frac{\theta_m^{\frac{1}{T}}}{\sum_{j\in I}\theta_j^{\frac{1}{T}}}\right\}\\
            &= - T^{-2} \Bigg[ \frac{\theta_M^{\frac{1}{T}}}{\sum_{j\in I} \theta_j^{\frac{1}{T}}} \sum_{i\in I} \Bigg( \frac{\theta_i^{\frac{1}{T}}}{\sum_{j\in I} \theta_j^{\frac{1}{T}}} \left(\log(\theta_M) - \log(\theta_i)\right) \Bigg)\\&\qquad\qquad -\frac{\theta_m^{\frac{1}{T}}}{\sum_{j\in I} \theta_j^{\frac{1}{T}}} \sum_{i\in I} \Bigg( \frac{\theta_i^{\frac{1}{T}}}{\sum_{j\in I} \theta_j^{\frac{1}{T}}} \left(\log(\theta_m) - \log(\theta_i)\right) \Bigg)\Bigg].
        \end{aligned}
    \end{equation*}
    The sums are convex combinations of non-negative terms for the first and non-positive for the second summation, as $\log(\theta_M) \ge \log(\theta_i)$ and $\log(\theta_m) \le \log(\theta_i)$ for all $i\in I$. Further, by Assumption~\ref{ass:temperature} $\max_{i\in I}\theta_i > \min_{i\in I}\theta_i$, therefore the sums are strictly positive and strictly negative respectively, and $\frac{\partial}{\partial T}V_I(\tau(\theta)) < 0$. Then, fixing some specific temperature $\overline T$:
    \begin{equation}\label{eq:integral}
        V_I(\tau(\theta)) - V_I(\theta) = \int_{T=1}^{\overline T} \frac{\partial}{\partial T}V_I(\tau(\theta)) \diff T,
    \end{equation}
    where~\eqref{eq:integral} is negative for $\overline T > 1$, and positive for $\overline T < 1$.
    
\end{proof}

\subsection{Closed-loop learning}
For a training dataset $D(\ell) = \left\{Y(1), Y(2), \dots, Y(\ell)\right\}$, ``learning'' a generative model is usually framed in a maximum-likelihood sense, i.e., we want to find the set of parameters $w\in\R^p$ whose associated probability vector $\Theta = \phi(w)$ maximizes the log-probability of the observed data:
{\color{black}\begin{equation}\label{eq:max_likelihood}
\begin{aligned}
    w^* &= \arg\max_{w\in\R^p}~ -\frac{1}{\ell}\sum_{i=1}^\ell \sum_{j=1}^n \mathds{1}_{Y(i) = \mathcal Y_j}\log\left(\phi(w)_j\right) \\
    &= \arg\max_{w\in\R^p}~ -\sum_{j=1}^n \left(\frac{1}{\ell}\sum_{i=1}^\ell \mathds{1}_{Y(i) = \mathcal Y_j}\right)\log\left(\phi(w)_j\right) \\
    &= \arg\max_{w\in\R^p}~ - \sum_{j=1}^n \Theta^*_j \log\left(\phi(w)_j\right)\\
    &= \arg\max_{w\in\R^p}~ H\left(\Theta^*, \phi(w)\right),
\end{aligned}
\end{equation}
% \begin{equation}
%     w^* = \arg\max_{w\in\R^p}~ \prod_{i=1}^\ell \phi(w)^\top Y(i).
% \end{equation}
% It is well known that this is equivalent to minimizing the ``empirical'' cross entropy of $\phi(w)$ with respect to the data:
% \begin{equation}\label{eq:max_likelihood}
% \begin{aligned}
%     w^* &= \arg\max_{\ell\in\R^p}~ -\frac{1}{\ell}\sum_{i=1}^\ell \log\left(\phi(w)^\top Y(i)\right) \\
%     &= \arg\max_{w\in\R^p}~ -\frac{1}{\ell}\sum_{i=1}^\ell \sum_{j=1}^n Y(i)_j\log\left(\phi(w)_j\right)\\
%     &= \arg\max_{w\in\R^p}~ - \sum_{j=1}^n \Theta^* \log\left(\phi(w)_j\right)\\
%     &= \arg\max_{w\in\R^p}~ H\left(\Theta^*, \phi(w)\right),
% \end{aligned}
% \end{equation}
where $\mathds{1}_{(\cdot)}$ is the indicator function (takes value 1 if the subscript expression is true and 0 otherwise), $H$ is the cross-entropy between probability vectors, and $\Theta^* = \frac{1}{\ell}\sum_{i=1}^\ell Y(i)$ is the ``empirical'' probability vector associated to the data $D$ (remember that with no loss of generality we take $\mathcal Y_j$ to be the $j$-th standard basis element of $\R^n$).} Therefore, although usually not expressed in this way, learning the model is also equivalent to minimizing the cross entropy with respect to $\Theta^*$. Note that if $\phi$ is surjective (where its codomain is $\Delta^n$), there always exists a $w^*$ such that $\phi(w^*) = \Theta^*$ and is thus the optimal solution of~\eqref{eq:max_likelihood}.

When $\phi$ is surjective, $\Theta = \Theta^*$, thus we study the behavior of $\Theta^*$ under the closed-loop learning dynamics. Consider the process~\eqref{eq:self-learning_model}, and let $\Theta^*(k) = \frac{1}{k}\sum_{i=1}^k Y(i)$ be the empirical probability vector corresponding to the dataset at time $k$, $D(k)$. When a new sample $Y(k+1)$ is generated by the model, $\Theta^*$ evolves as:
\begin{equation}
\begin{aligned}
    \Theta^*(k+1) &= \frac{1}{k+1}\sum_{i=1}^{k+1} Y(i)\\
    &= \frac{k}{k+1}\Theta^*(k) + \frac{1}{k+1}Y(k+1)\\
    &= \Theta^*(k) + \frac{1}{k+1}\left(Y(k+1) - \Theta^*(k)\right).
\end{aligned}
\end{equation}
Since $Y(k+1)$ is a random variable, we perform a Martingale decomposition~\cite{liptser2012theory}:
% \begin{equation*}
%     \begin{aligned}
%         \Theta^*(k+1) &= \Theta^*(k) + \left( \mathbb E\left[\Theta^*(k+1)|\mathcal F_k\right] - \Theta^*(k)\right)\\&\qquad\qquad + \left(\Theta^*(k+1) - \mathbb E\left[\Theta^*(k+1)|\mathcal F_k\right] \right) \\
%         &= \Theta^*(k) + \frac{1}{k+1}\left( \tau(\Theta(k)) - \Theta^*(k) \right) + U(k+1),
%     \end{aligned}
% \end{equation*}
\begin{equation*}
    \begin{aligned}
        \Theta^*(k+1) &= \Theta^*(k) + \frac{1}{k+1}\big(\mathbb E\left[Y(k+1)|\mathcal F_k\right] - \Theta^*(k) \\&\qquad\qquad\qquad\qquad + Y(k+1) - \mathbb E\left[Y(k+1)|\mathcal F_k\right]\big) \\
        &= \Theta^*(k) + \frac{1}{k+1}\left( \tau(\Theta(k)) - \Theta^*(k) + U(k+1)\right),
    \end{aligned}
\end{equation*}
where $U(k+1) = Y(k+1) - \mathbb E\left[Y(k+1)|\mathcal F_k\right]$ is a bounded Martingale difference sequence, and the second equality holds because \mbox{$\mathbb E\left[Y(k+1)|\mathcal F_k\right] = \tau(\Theta)$}.

Then, if the update function $f$ in~\eqref{eq:self-learning_model} is the maximum-likelihood optimization~\eqref{eq:max_likelihood}:
\begin{equation}
\begin{aligned}
    w(k+1) &= f(w(k), D(k+1))\\
    &= \arg\max_{w\in\R^p} H\left(\Theta^*(k+1), \phi(w)\right),
\end{aligned}
\end{equation}
and $\phi$ is surjective on $\Delta^n$, at each step $\Theta(k) = \phi(w(k)) = \Theta^*(k)$, leading to the following dynamics:
\begin{equation*}
    \Theta(k+1) = \Theta(k) + \frac{1}{k+1}\left(\tau(\Theta(k)) - \Theta(k) + U(k+1)\right).
\end{equation*}

An actual model is unlikely to have enough expressivity as to represent \emph{any} element of $\Delta^n$, especially for very high dimension $n$. However, we assume the model is able to approximate a probability vector with some small finite accuracy:
\begin{assumption}\label{ass:delta}
    There exists some $\delta\in\R_{\ge 0}$ such that for all $k\in\mathbb Z_{\ge 0}$ the process~\eqref{eq:self-learning_model} satisfies:
    \begin{equation}
    \begin{aligned}\label{eq:optim_error}
        \left\Vert \Theta^*(k) - \Theta(k)\right\Vert = \left\Vert \Theta^*(k) - \phi(w^*(k))\right\Vert \le \delta.
    \end{aligned}
    \end{equation}
\end{assumption}
Under Assumption~\ref{ass:delta}, the dynamics of $\Theta^*(k)$ obey:
\begin{equation}\label{eq:theta_star_model}
    \Theta^*(k+1) = \Theta^*(k) + \frac{1}{k+1}\left(\tau(\Theta(k)) - \Theta^*(k) + U(k+1)\right).
\end{equation}
If we now define the perturbation $\varepsilon(k) = \tau(\Theta)-\tau(\Theta^*)$, recalling that a continuous function on a compact set is uniformly continuous, we have the following inequality, where $\eta$ is the modulus of continuity of $\tau$:
\begin{equation}
    \norm{\varepsilon(k)} = \lVert\tau(\Theta)-\tau(\Theta^*)\rVert \le \eta(\delta).
\end{equation}
Then, the dynamics of $\Theta^*(k)$ become a stochastic approximation (see~\cite{borkar2009stochastic}) of the form:
\begin{equation}\label{eq:process}
\begin{aligned}
    \Theta^*(k+1) &= \Theta^*(k) + \frac{1}{k+1}(\tau\left(\Theta^*(k)\right) - \Theta^*(k) + \varepsilon(k) + U(k+1)).
\end{aligned}
\end{equation}
If we understand the behavior of~\eqref{eq:process}, we automatically understand the behaviour of $\Theta(k)$, since by~\eqref{eq:optim_error} $\Theta(k)$ is always within a distance $\delta$ from $\Theta^*(k)$.

\section{Main Results}\label{sec:theoretical_results}
We now present the main results describing the asymptotic behavior of~\eqref{eq:process} (and hence of $\Theta(k)$ up to error $\delta$). This asymptotic behavior  is important, as it determines the long-term composition of the dataset $D(k)$, and how much it may diverge from its initial distribution. The results are different depending on which of the three conditions (2.a, 2.b, 2.c) in Assumption~\ref{ass:temperature} is satisfied by $\tau$, therefore we split the analysis in three different cases.

\subsection{Identity temperature leads to Martingale-like behavior}
This case is the most straightforward and does not require any machinery beyond standard analysis tools of stochastic processes. If condition 2.a of Assumption~\ref{ass:temperature} is satisfied, the stochastic process~\eqref{eq:process} reduces to:
\begin{equation}\label{eq:reduced_process}
    \Theta^*(k+1) = \Theta^*(k) + \frac{1}{k+1}\left(U(k+1) + \varepsilon(k)\right),
\end{equation}
and we state the following formal result.
\begin{theorem}
\label{Theorem1}
    Consider the closed-loop learning stochastic process~\eqref{eq:self-learning_model}, where $\tau$ satisfies Assumption~\ref{ass:temperature} with property 2.a ($\tau$ is the identity), and Assumption~\ref{ass:delta}. Then it holds that:
    \begin{equation*}
        \mathbb E\left[\Theta(k+1)~|~\mathcal F_{\ell}\right] = \Theta(\ell) + \sum_{i=\ell}^k\frac{1}{i+1}\mathbb E\left[\varepsilon(i)~|~\mathcal F_\ell\right] + e_\delta(k+1),
    \end{equation*}
    where $e_\delta:\mathbb Z_{\ge 0}\to \R^n$ is such that $\norm{e_\delta(k)} \le \delta$. In addition, if $\varepsilon$ is a Martingale difference sequence, there is a constant $b\in\R_{\ge 0}$ such that the asymptotic variance is bounded as:
    \begin{equation}
        \lim_{k\to\infty} \var\left(\Theta(k)-\Theta(\ell)\right) \le b\left(\sum_{i=\ell}^\infty\left(\frac{1}{i+1}\right)^2 + \delta^2\right).
    \end{equation}
\end{theorem}
\begin{proof}
    If $\tau$ is the identity function, $\Theta^*: \mathbb Z_{\ge 0} \to \Delta^n$ satisfies~\eqref{eq:reduced_process}. Then, because $U$ is a Martingale difference sequence:
    \begin{equation*}
        \begin{aligned}
            \mathbb E\left[\Theta^*(k+1)~|~\mathcal F_k\right] &= \mathbb E\left[\Theta^*(k) + \frac{1}{k+1}\left(U(k+1) + \varepsilon(k)\right)~|~\mathcal F_k\right]\\
            &= \Theta^*(k) + \frac{1}{k+1}\varepsilon(k),
        \end{aligned}
    \end{equation*}
    and by the tower property of expectation:
    \begin{equation*}
        \begin{aligned}
            \mathbb E\left[\Theta^*(k+1)~|~\mathcal F_{k-1}\right] &= \mathbb E\left[\mathbb E\left[\Theta^*(k+1)~|~\mathcal F_k\right]~|~\mathcal F_{k-1}\right]\\
            &= \mathbb E\left[\Theta^*(k) + \frac{1}{k+1}\varepsilon(k)~|~\mathcal F_{k-1}\right]\\
            &= \Theta^*(k-1) + \frac{1}{k}\varepsilon(k-1) + \frac{1}{k+1}\mathbb E\left[\varepsilon(k)~|~\mathcal F_{k-1}\right].
        \end{aligned}
    \end{equation*}
    Finally, by recursion we arrive at:
    \begin{equation*}
        \mathbb E\left[\Theta^*(k+1)~|~\mathcal F_{\ell}\right] = \Theta^*(\ell) + \sum_{i=\ell}^k\frac{1}{i+1}\mathbb E\left[\varepsilon(i)~|~\mathcal F_\ell\right],
    \end{equation*}
    and the statement is obtained once we take into account that $\Theta(k)$ is always within a distance $\delta$ from $\Theta^*(k)$.
    
    In addition, if $\varepsilon$ is a Martingale difference sequence, the whole $\Theta^*(k)$ process reduces to a sum of bounded Martingale differences, and it immediately follows that its variance is bounded by a term of the order of the converging sum $\sum_{i=\ell}^\infty\left(i+1\right)^{-2}$.
\end{proof}
Theorem~\ref{Theorem1} states that in this case $\Theta(k)$ is essentially a Martingale biased by the perturbation $\varepsilon$. In general the asymptotic behavior can be arbitrary as it is dominated by the behavior of $\varepsilon(k)$. However, if $\varepsilon$ is also a Martingale difference sequence, with probability one $\Theta(k)$ will only drift a finite amount from its initial value. The magnitude of this drift depends on the converging sum $\sum_{i=\ell}^\infty\left({i+1}\right)^{-2}$, which is smaller the greater the initial dataset size $\ell$ is. {\color{black} This shows that in this case the initial data distribution of the dataset is not necessarily lost. However, this condition requires hard to verify assumptions on the training behavior (captured by $\varepsilon(k)$) and, if the training dataset is shared by multiple generative models, that no model is biasing their own sampling via temperature. We consider this especially unlikely for data on the public web. From a control perspective, the behavior with identity temperature is similar to that of a marginally stable system, and any arbitrarily small perturbation $\varepsilon(k)$ can destabilize it.}

\subsection{High temperature leads to uniformly generated data}
While the analysis of the previous case is relatively straightforward, we need to introduce additional machinery for the remaining two. Let us define a vector field \mbox{$F:\Delta^n \to T\Delta^n$} over the probability simplex as $F(\theta) = \tau(\theta) - \theta$. The behavior of stochastic approximations in the long-term approaches that of a continuous-time ODE (ordinary differential equation) or differential inclusion (see~\cite{borkar2009stochastic}). Thus, under our assumptions, the limit sets of $\Theta^*(k)$ in~\eqref{eq:process} are determined by the attractors of:
\begin{equation}\label{eq:ode}
    \dot\theta(t) = F(\theta(t)) + \varepsilon(t).
\end{equation}
Then, we introduce two families of sets, parameterized by $a\in\R_{\ge 0}$ and an index set $I\subseteq\{1,2,\dots,n\}$, that will be used to characterize the attractors and basins of attraction of~\eqref{eq:ode}:
\begin{equation}\label{eq:omega_sets}
\begin{aligned}
    \underline\Omega_I(a) &= \left\{\theta\in\Delta^n~\big|~V_I(\theta) \le a\right\}\\
    \overline\Omega_I(a) &= \left\{\theta \in \Delta^n~|~\min_{i\in I}\theta_i \le a\right\}.
\end{aligned}
\end{equation}
With respect to the subset of variables indexed by $I$, the set $\underline\Omega_I(a)$ is a compact neighborhood of the uniform probability vector (all $\theta_i$ are equal), while $\overline\Omega_I(a)$ is a compact neighborhood of the boundary of the probability simplex (at least one $\theta_i$ is zero). We can now state the following lemma:
\begin{lemma}\label{lem:high_temp}
    Consider the continuous-time ODE:
    \begin{equation}
        \dot\theta(t) = F(\theta(t)) + \varepsilon(t),
    \end{equation}
    where $\theta\in\Delta^n$, $\norm{\varepsilon}\le\eta\in\R_{\ge 0}$, and $F(\theta)=\tau(\theta)-\theta$. Let $\tau:\Delta^n\to\Delta^n$ satisfy Assumption~\ref{ass:temperature} and property 2.b (contractive $\tau$). Then, for any index set $I\subseteq\{1,2,\dots,n\}$ there exists $\kappa\in\mathcal K$ such that for any $t_0\in\R$:
    \begin{equation}
        \theta(t_0) \not\in \overline\Omega_I(\kappa(\eta))
    \end{equation}
    implies:
    \begin{equation}\label{eq:high_temp_set_limit}
    \theta(t) \xrightarrow[t\to\infty]{} \underline\Omega_I(\kappa(\eta)).
        % \lim_{t\to\infty} d\left(\theta(t), \left\{ \theta\in\Delta^n~|~\theta_i=\theta_j,~i,j\in I\right\}\right) \le \kappa(\eta).
    \end{equation}
\end{lemma}
\begin{proof}
    Consider a point $\theta\in\Delta^n$, and an index set $I\subseteq\{1,2,\dots,n\}$ where $\min_{i\in I} \theta_i > 0$ (note that the lemma only makes a claim for index sets such that $\theta(t_0) \not\in \overline\Omega_I(\kappa(\eta))$, which satisfy this condition).
    Let $M,m\in I$ be the (not necessarily unique) indices of the greatest and smallest elements of $\left\{\theta_i~|~i\in I\right\}$. The time derivative of $V_I$ at some point $\theta$ is:
    \begin{equation*}
        \begin{aligned}
            \dot V_I(\theta) &= \frac{\diff}{\diff t}\left\{ \left(\sum_{i\in I}\theta_i\right)^{-1} \left(\max_{i\in I} \theta_i - \min_{i\in I} \theta_i \right)\right\} \\
            &= -\left(\sum_{i\in I}\theta_i\right)^{-2}\left(\sum_{i\in I}\dot\theta_i\right) \left(\theta_M - \theta_m \right) +\left(\sum_{i\in I}\theta_i\right)^{-1}\left(\dot\theta_M - \dot\theta_m \right)\\
            &\le \frac{\sum_{i\in I} \tau(\theta)_i}{\sum_{i\in I} \theta_i} \left(\frac{\tau(\theta)_M-\tau(\theta)_m}{\sum_{i\in I} \tau(\theta)_i} - \frac{\theta_M-\theta_m}{\sum_{i\in I} \theta_i}\right) + b\eta\\
            &\le c\left(V_I(\tau(\theta)) - V_I(\theta)\right) + b\eta,
            % &= \frac{\diff}{\diff t}\left\{ \left(\sum_{i\in I}\theta_i\right)^{-1} \left(\max_{i\in I} \theta_i - \min_{i\in I} \theta_i \right)\right\} \\
            % &= \left(F(\theta) + \varepsilon\right)_M - \left(F(\theta) + \varepsilon\right)_m \\
            % &\le \left(\tau(\theta) - \theta\right)_M - \left(\tau(\theta) - \theta\right)_m + 2\norm{\varepsilon} \\
            % &\le \tau(\theta)_M - \tau(\theta)_m - (\theta_M - \theta_m) + 2\eta \\
            % &= \max_{i\in I} \tau(\theta)_i - \min_{i\in I} \tau(\theta)_i - \left(\max_{i\in I} \theta_i - \min_{i\in I} \theta_i\right) + 2\eta \\
            % &= V_I(\tau(\theta)) - V_I(\theta) + 2\eta,
        \end{aligned}
    \end{equation*}
    where the second equality holds because $\tau$ preserves the indices of the maximum and minimum elements of $\theta$ by its order preserving property, guaranteeing that the derivative of $V_I$ is well-defined. In the last two inequalities $b,c\in\R_{> 0}$ are finite positive constants, as the sums that appear are always positive and bounded from above and below.

    We now seek to establish that for $\eta$ small enough, there is some $a\in\R_{\ge 0}$ such that the set difference $\mathcal B(a) = \Delta^n \setminus (\underline\Omega_I(a) \cup \overline\Omega_I(a))$ is a basin of attraction for $\underline\Omega_I(a)$. Remember that by property 2.b in Assumption~\ref{ass:temperature}, $V_I(\tau(\theta)) - V_I(\theta) < 0$ over points that lie in $\mathcal B(0)$ (i.e., points that are not the uniform probability vector over $I$  and with no zero elements). Consider the closure of the set $\mathcal B(a)$ for some $a>0$, and let us define the ``worst-case'' decrease of $V_I$ for points in $\operatorname{clo}(\mathcal B(a))$ as:
    \begin{equation}
        \beta(a) = \sup_{\theta\in\operatorname{clo}(\mathcal B(a))} \left\{V_I(\tau(\theta)) - V_I(\theta)\right\}.
    \end{equation}
    By continuity, and because for $a > 0$, the set $\mathcal B(a)$ always excludes an open neighborhood of the region where $V_I(\tau(\theta))-V_I(\theta) = 0$, we have that $\beta(a) < 0$. Then, for any $\mathcal B(a) \ne \varnothing$, as long as $\eta < -cb^{-1}\beta(a)$, the term $\dot V_I(\theta)$ is strictly negative for all $\theta\in\mathcal B(a)$.

    Observe that $\beta(\cdot)$ is decreasing, since $\mathcal B(a_2) \subseteq \mathcal B(a_1)$ for $a_2 > a_1$, and $\beta(0) = 0$. Then, we can define $\kappa\in\mathcal K$ as any strictly increasing continuous function such that:
    \begin{equation}
        \kappa(\eta) > \inf \left\{a\in\R_{\ge 0}~\Big|~\eta < -cb^{-1}\beta(a)\right\},
    \end{equation}
    for any $\eta$ where this $\inf$ is finite, which is guaranteed for any $\eta$ smaller than some finite threshold. Then, any initial condition in the basin $\theta(t_0)\in \mathcal B(\kappa(\eta))$ guarantees that $\dot V_I(\theta)<0$, and $\theta(t)$ will converge to $\underline\Omega(\kappa(\eta))$.
\end{proof}

This lemma essentially claims that the solutions of the limiting ODE~\eqref{eq:ode} over a subset of indices whose elements are sufficiently away from zero converge to a neighborhood of the uniform probability vector conditioned over that subset of indices. This key result allows us to then claim the following theorem:

\begin{theorem}\label{theo:high_temperature}
    Consider the closed-loop learning stochastic process~\eqref{eq:self-learning_model}, where $\tau$ satisfies Assumption~\ref{ass:temperature} with property 2.b ($\tau$ is contractive), and Assumption~\ref{ass:delta}.
    Then, for any index set $I\subseteq\{1,2,\dots,n\}$ there exists $\sigma\in\mathcal K$ such that:
    \begin{equation}
        \Theta(k) \xrightarrow[k\to\infty]{} \underline\Omega_I(\sigma(\delta)) \cup \overline\Omega_I(\sigma(\delta))\qquad\text{w.p. 1}.
    \end{equation}
    % where:
    % \begin{equation}
    %     \begin{aligned}
    %         \underline\Omega_I &= \left\{\theta\in\Delta^n~\big|~V_I(\theta) \le \sigma(\delta)\right\}\\
    %         \overline\Omega_I &= \left\{\theta \in \Delta^n~|~\min_{i\in I}\theta_i \le \sigma(\delta)\right\}.
    %     \end{aligned}
    % \end{equation}
    Further, for a given time $k_0\in\mathbb Z_{\ge 0}$:
    \begin{equation*}
        \operatorname{P}\left(\Theta(k)\xrightarrow[k\to\infty]{}\underline\Omega_I(\sigma(\delta))~\Big|~\Theta(k_0)\not\in\overline\Omega_I(\sigma(\delta))\right) \xrightarrow[k_0\to\infty]{} 1.
    \end{equation*}
\end{theorem}
\begin{proof}
    Under these conditions, $\Theta^*(k)$ satisfies the stochastic approximation~\eqref{eq:theta_star_model}. Then, by Corollary 4 (Chapter 5) of~\cite{borkar2009stochastic}, the iterates of $\Theta^*(k)$ converge a.s. to a closed connected internally chain transitive invariant set of the continuous-time differential inclusion:
    \begin{equation}\label{eq:inclusion}
        \dot\theta(t) \in \left\{x\in\Delta^n~|~\left\Vert x - F(\theta(t))\right\Vert \le \eta(\delta)\right\},
    \end{equation}
    where $F(\theta) = \tau(\theta)-\theta$.
    
    Lemma~~\ref{lem:high_temp} characterizes the solutions of~\eqref{eq:inclusion}, thus, letting $a=\kappa(\eta(\delta))$, any solution that enters $\mathcal B(a) = \Delta^n \setminus (\underline\Omega_I(a) \cup \overline\Omega_I(a))$ will converge to $\underline\Omega_I(a)$.
    % \begin{equation}
    %     {\underline\Omega}'_I = \left\{ \theta\in\Delta^n~\big|~V_I(\theta) \le \kappa(\eta(\delta))\right\}
    % \end{equation}
    Then, the set $\mathcal L$ of limit points of~\eqref{eq:inclusion} is contained in:
    \begin{equation}
        \mathcal L \subseteq \underline\Omega_I(a) \cup \overline\Omega_I(a),
    \end{equation}
    In turn, any internally chain transitive set $\mathcal A$ of~\eqref{eq:inclusion} is contained in the closure of the limit set $\mathcal L$:
    \begin{equation*}
    \begin{aligned}
        \mathcal A &\subseteq \operatorname{clo}(\mathcal L)\subseteq \operatorname{clo}\left(\underline\Omega_I(a) \cup \overline\Omega_I(a)\right) = \underline\Omega_I(a) \cup \overline\Omega_I(a),
    \end{aligned}
    \end{equation*}
    so that w.p. 1: $\Theta^*(k)\xrightarrow[k\to\infty]{} \mathcal A \subseteq \underline\Omega_I(a) \cup \overline\Omega_I(a)$. Finally, because $\left\Vert\Theta(k) - \Theta^*(k)\right\Vert \le \delta$, we know that $$\Theta(k)\xrightarrow[k\to\infty]{} \underline\Omega_I(\sigma(\delta)) \cup \overline\Omega_I(\sigma(\delta)),$$ with $\sigma(\delta) = a + 2\delta = \kappa(\eta(\delta)) + 2\delta$.
    
    Because any point $\theta\not\in\overline\Omega_I(a)$ either is in $\underline\Omega_I(a)$, or is in $\mathcal B(a)$, which is an open basin of attraction for $\underline\Omega_I(a)$, the last result holds by~\cite[Theorem III.2]{yaji2019analysis}, noting that the stochastic recursion~\eqref{eq:theta_star_model} satisfies assumptions A1-3. In fact, the theorem gives explicit bounds for this probability.
\end{proof}
The result essentially states that as long as the learning model is sufficiently powerful, and the training sufficiently good (low $\delta$), every set of elements of $\Theta$ will either converge to a neighborhood of the uniform probability vector conditioned over that subset, or at least one of its elements remains trapped close to zero. The reason this second possibility can occur, is that the vector field induced by the temperature function may vanish at the boundary $\partial\Delta^n$, so that some small perturbation $\varepsilon$ over the process can keep it trapped. However, the greater the size of the initial dataset, the higher the probability of the process converging towards the uniform probability.

Either way, regardless of the size of the initial dataset (for small $\delta$, and iterating the result over all sets $I$), any information it originally contained is lost as $k\to\infty$. Some subset of output probabilities will approach zero, and the rest will approach their (conditioned) uniform distribution. In summary, in the limit, as $k$ increases, the set of possible outputs is partitioned into the outcomes that will (almost) never be generated, and the outcomes that will be (almost) uniformly generated.

\begin{remark}
    {\color{black} While we are considering the setting where there is only a fixed amount of external initial data $D(\ell)$, our results hold even when some limited amount of external data is introduced at each training iteration.\footnote{\color{black} These two scenarios are analogous to the ``synthetic augmentation loop'' and ``fresh data loop'' in~\cite{alemohammad2023self}. In the first one, the amount of external data is fixed at the start, so that over time the proportion of synthetic data dominates the dataset. In the second one, some proportion $\lambda$ of external data is introduced at each step (this may be additional copies of samples from the initial dataset), guaranteeing that the proportion of synthetic data is always less than $1-\lambda$.}}
    To see this, let $\lambda\in[0,1]$ be the fraction of external data we introduce at each time step. Then~\eqref{eq:process} becomes:
    \begin{equation*}
    \begin{aligned}
        \Theta^*(k+1) &= \Theta^*(k) + \frac{1}{k+1}\Big(\tau\left(\Theta^*(k)\right) - \Theta^*(k) +\\ &\quad + \lambda\left(\widetilde Y(k) - \tau\left(\Theta^*(k)\right) - \Theta^*(k)\right) + \varepsilon(k) + U(k+1)\Big),
    \end{aligned}
    \end{equation*}
    where $\widetilde Y(k)$ is the external data point at time $k$. Because $\widetilde Y$ is bounded, the term $\lambda(\widetilde Y(k) - \tau\left(\Theta^*(k)\right) - \Theta^*(k))$ is bounded and can be absorbed into $\varepsilon(k)$.
\end{remark}

\begin{remark}
    It may be the case that for very high dimensional outputs ($n >> 1$), the assumption that $\delta$ is sufficiently small for every output probability is unrealistic. However, the assumption may still hold over a ``coarse-grained'' model, where we group outputs $\left\{ \mathcal Y_1, \dots,\mathcal Y_n\right\}$ into a set of \mbox{$m < n$} categories $\left\{\widehat {\mathcal Y}_1, \dots, \widehat {\mathcal Y}_m\right\}$. In this case the result would reduce to some categories disappearing, and others appearing uniformly randomly as $k\to\infty$.
\end{remark}

\subsection{Low temperature leads to mode collapse}
The low temperature case is {\color{black} identical to}
the high temperature one, but with the roles of $\underline\Omega$ and $\overline\Omega$ swapped, so we only state the corresponding theorem. {\color{black} In the proof,} the direction of the Lyapunov inequalities is swapped and the sign inverted.

\begin{theorem}\label{theo:low_temperature}
    Consider the closed-loop learning stochastic process~\eqref{eq:self-learning_model}, where $\tau$ satisfies Assumption~\ref{ass:temperature} with property 2.c ($\tau$ is expanding), and Assumption~\ref{ass:delta}.
    Then, for any index set $I\subseteq\{1,2,\dots,n\}$ there exists $\sigma\in\mathcal K$ such that:
    \begin{equation}
        \Theta(k) \xrightarrow[k\to\infty]{} \underline\Omega_I(\sigma(\delta)) \cup \overline\Omega_I(\sigma(\delta))\qquad\text{w.p. 1}.
    \end{equation}
    % where:
    % \begin{equation}
    %     \begin{aligned}
    %         \underline\Omega_I &= \left\{\theta\in\Delta^n~\big|~V_I(\theta) \le \sigma(\delta)\right\}\\
    %         \overline\Omega_I &= \left\{\theta \in \Delta^n~|~\min_{i\in I}\theta_i \le \sigma(\delta)\right\}.
    %     \end{aligned}
    % \end{equation}
    Further, for a given time $k_0\in\mathbb Z_{\ge 0}$:
    \begin{equation*}
        \operatorname{P}\left(\Theta(k)\xrightarrow[k\to\infty]{}\overline\Omega_I(\sigma(\delta))~\Big|~\Theta(k_0)\not\in\underline\Omega_I(\sigma(\delta))\right) \xrightarrow[k_0\to\infty]{} 1.
    \end{equation*}
\end{theorem}
Just like for the high temperature case, any information in the original dataset is lost, with data generated by the asymptotic behavior of $\Theta(k)$ dominating the dataset. Unlike the high temperature case, with high probability $\Theta$ will converge to a region where most outputs have very low probability mass, and only a few outputs are likely to be sampled. {\color{black} In the limit of $\delta\to0$, for almost every initial condition the generative probabilities of every output approach zero except for a single output element, that will completely dominate the dataset.}

\section{Conclusion}
We have shown that when a generative model is trained on the data it generates, and this generation is biased by temperature (no matter how small the biasing), there is a dichotomy between the accuracy of the learning model and preserving the initial distribution of the dataset {\color{black} unless that initial dataset is preserved and re-injected purposefully}. A model capable of accurately reproducing the distribution of a training dataset (low $\delta$ in~\eqref{eq:optim_error}) will inevitably degenerate into never producing some outputs and producing the rest uniformly randomly.

Our theoretical analysis adds to the increasing concern about data self-ingestion, especially in the current age where {\color{black} large scale deep networks are trained on data scraped from the internet, and} data generated by these models inevitably finds its way back to their training processes.

\bibliographystyle{amsplain}
\bibliography{bibliography}

\providecommand{\bysame}{\leavevmode\hbox to3em{\hrulefill}\thinspace}
\providecommand{\MR}{\relax\ifhmode\unskip\space\fi MR }
\providecommand{\MRhref}[2]{%
  \href{http://www.ams.org/mathscinet-getitem?mr=#1}{#2}
}
\providecommand{\href}[2]{#2}
\begin{thebibliography}{10}

\bibitem{alemohammad2023self}
S.~Alemohammad et~al., \emph{Self-consuming generative models go mad}, International Conference on Learning Representations (ICLR) (2024).

\bibitem{bertrand2023stability}
Q.~Bertrand et~al., \emph{On the stability of iterative retraining of generative models on their own data}, International Conference on Learning Representations (ICLR) (2024).

\bibitem{biggio2012poisoning}
B.~Biggio, B.~Nelson, and P.~Laskov, \emph{Poisoning attacks against support vector machines}, Proceedings of the 29th International Conference on Machine Learning, ICML'12, 2012, p.~1467–1474.

\bibitem{borkar2009stochastic}
V.~S. Borkar, \emph{Stochastic approximation: A dynamical systems viewpoint}, Texts and Readings in Mathematics, vol.~48, Springer, 2023.

\bibitem{cao2024survey}
H.~Cao et~al., \emph{A survey on generative diffusion models}, IEEE Transactions on Knowledge and Data Engineering (2024).

\bibitem{carlini2023poisoning}
N.~Carlini et~al., \emph{Poisoning web-scale training datasets is practical}, IEEE Symposium on Security and Privacy (SP), 2024.

\bibitem{koh2021wilds}
P.~W. Koh et~al., \emph{Wilds: A benchmark of in-the-wild distribution shifts}, Proceedings of the 29th International Conference on Machine Learning, ICML'21, 2021, pp.~5637--5664.

\bibitem{liptser2012theory}
R.~Liptser and A.~N. Shiryayev, \emph{Theory of martingales}, Mathematics and its Applications, vol.~49, Springer, 1989.

\bibitem{martinez2023combining}
G.~Mart{\'\i}nez et~al., \emph{Combining generative artificial intelligence (ai) and the internet: heading towards evolution or degradation?}, arXiv preprint arXiv:2303.01255 (2023).

\bibitem{martinez2023towards}
\bysame, \emph{Towards understanding the interplay of generative artificial intelligence and the internet}, arXiv preprint arXiv:2306.06130 (2023).

\bibitem{shumailov2023curse}
I.~Shumailov et~al., \emph{The curse of recursion: Training on generated data makes models forget}, arXiv preprint arXiv:2305.17493 (2023).

\bibitem{thanh2020catastrophic}
H.~Thanh-Tung and T.~Tran, \emph{Catastrophic forgetting and mode collapse in gans}, International Joint Conference on Neural Networks (IJCNN), 2020.

\bibitem{touvron2023llama}
H.~Touvron et~al., \emph{Llama 2: Open foundation and fine-tuned chat models}, arXiv preprint arXiv:2307.09288 (2023).

\bibitem{yaji2019analysis}
V.~G. Yaji and S.~Bhatnagar, \emph{Analysis of stochastic approximation schemes with set-valued maps in the absence of a stability guarantee and their stabilization}, IEEE Transactions on Automatic Control \textbf{65} (2019), no.~3, 1100--1115.

\bibitem{zhao2023survey}
W.~X. Zhao et~al., \emph{A survey of large language models}, arXiv preprint arXiv:2303.18223 (2023).

\end{thebibliography}

\end{document}